\documentclass{article}

\usepackage[preprint]{neurips_2026}
\usepackage{array}
\usepackage[utf8]{inputenc} 
\usepackage[T1]{fontenc}    
\usepackage{hyperref}       
\usepackage{url}            
\usepackage{booktabs}       
\usepackage{amsfonts}       
\usepackage{nicefrac}       
\usepackage{microtype}      
\usepackage{xcolor}         
\usepackage{lipsum}
\usepackage{graphicx}
\usepackage{doi}
\usepackage{amsmath}
\usepackage{subcaption}
\usepackage{tikz}

\usetikzlibrary{arrows.meta,positioning,fit,calc}
\usetikzlibrary{decorations.pathreplacing} 

\renewcommand{\vec}[1]{\textbf{#1}}
\newcommand{\norm}[1]{||#1||}
\newcommand{\A}{\vec{A}}
\newcommand{\C}{\vec{C}}
\newcommand{\R}{\mathbb{R}}
\newcommand{\B}{\vec{B}}
\usepackage{algorithm}
\usepackage{algpseudocode}
\usepackage{amssymb}
\usepackage{amsthm} 
\newcommand{\abs}[1]{\lvert #1 \rvert}

\newtheorem{theorem}{Theorem}
\newtheorem{lemma}{Lemma}

\newtheorem{corollary}{Corollary}

\newtheorem{conjecture}{Conjecture}

\title{The Power of Second Order Methods for \\ Sequence Preconditioning}
\author{%
Annie Marsden \thanks{Google DeepMind} 
  \And
  Elad Hazan $^*$ \thanks{Princeton University}
}

\begin{document}

\maketitle

\begin{abstract}
Sequence prediction methods for linear dynamical systems with long memory, i.e. marginally stable systems, typically achieve regret that grows linearly with the hidden dimension of the underlying generative model. While many methods have been developed to address this regime with varying success, we show that simply using the second-order Vovk-Azoury-Warmuth (VAW) algorithm to learn a short autoregressive-with-inputs (ARX) model achieves astoundingly strong results: for bounded sequential data from a marginally-stable linear dynamical system with spectra in the complex disk except for angular wedge of width $\delta$ around the negative real axis, this algorithm achieves dimension-free regret $\mathcal{O}\!\left( \delta^{-4}  \log^2 T \right)$.  
These bounds are state-of-the-art to our knowledge. 

The key components for our result come from 1) using the theory of ``Universal Sequence Preconditioning'' (USP) \cite{marsdenuniversal} to prove the existence of an optimal setting of autoregressive coefficients, 2) the application of VAW which takes better advantage of the memory compression provided by USP, and 3) the analysis of Faber polynomials on circular sectors to extend these results to systems with complex spectra.  
\end{abstract}

\section{Introduction}

A central challenge in sequence prediction is the simultaneous presence of long memory and high hidden dimension of the underlying generative process. The canonical theoretical model to study in this setting is a marginally stable linear dynamical system. Most methods applied to this setting result in a regret bound which grows linearly with hidden dimension, see for example \cite{hazan2018generallds, kailath1980linear, chen1999linear,hazan2016introduction, simchowitz2019semiparametric, tsiamis2019stochsysid, lee2022improved, bakshi2023newapproach}. 
Universal Sequence Preconditioning (USP) \cite{marsdenuniversal} was the first method, to our knowledge, which simultaneously achieved dimension-free sublinear regret, even in the presence of hidden transition matrices with complex eigenvalues. The method they develop is to convolve the original signal with the coefficients of the Chebyshev polynomial. This construction achieves an exponential compression of memory, reducing the effective memory to be logarithmic in sequence length (or horizon). 
However, this comes at a cost: the Chebyshev coefficients grow exponentially large with the degree, resulting in large optimal diameter.  
This creates a mismatch with first-order learning methods, which scale linearly with optimal diameter, yielding suboptimal regret of $\mathcal{O}(T^{11/13})$ \cite{marsdenuniversal}. Moreover, while their work did apply to hidden transition matrices with complex eigenvalues, they required the assumption that the complex argument of the eigenvalues vanish at the rate $o(1/\mathrm{poly} \log T)$.

In this work, we take the insight of memory compression from Universal Sequence Preconditioning and use it to design a stronger algorithm which achieves dimension-free and poly-logarithmic regret for a much broader class of linear dynamical systems. This is an improvement upon all other known methods, illustrated more clearly in Table~\ref{table:shalom}.

\begin{table}[ht]
\centering
\small
\renewcommand{\arraystretch}{1.5} 
\begin{tabular}{|>{\centering\arraybackslash}p{0.12\linewidth}|>{\centering\arraybackslash}p{0.12\linewidth}|>{\centering\arraybackslash}p{0.13\linewidth}|p{0.45\linewidth}|}
\hline
\textbf{Dependence on $T$} & \textbf{Dependence on $d$} & \textbf{Complex Angle $\phi$} & \textbf{Methods} \\ \hline
& & & \textbf{Direct prediction methods:} Hazan et al. (2018) \cite{hazan2018generallds}, Cayley-Hamilton and regret minimization (folklore, e.g. \cite{kailath1980linear,chen1999linear,hazan2016introduction}) \\ 
$T^{O(1)}$ & $d^{O(1)}$ & $\pi$ &   \textbf{Recovery methods:} Simchowitz et al. (2019) \cite{simchowitz2019semiparametric}; Tsiamis \& Pappas (2019) \cite{tsiamis2019stochsysid}; Lee (2022) \cite{lee2022improved}; Bakshi et al. (2023) \cite{bakshi2023newapproach} \\ \hline
$T^{O(1)}$ & $(\log d)^{O(1)}$ &  $1/\log(T)^{O(1)}$ & Marsden \& Hazan (2025) \cite{marsden2025dimensionfree} \\ \hline
$(\log T)^{O(1)}$ & $(\log d)^{O(1)}$ & $\pi-\delta$ & \textbf{Ours} \\ \hline
\end{tabular}
\vspace{1mm}
\caption{\label{table:shalom} Cumulative prediction error bounds for \emph{marginally stable} linear dynamical systems with \emph{asymmetric} transition matrices. For ``Complex Angle $\phi$'', the results only apply when the hidden transition matrix has spectra in the wedge $W_{\phi}=\{z\in\mathbb{C}: |z|\le1,\ |\arg(z)|\le\phi\}$.}
\end{table}

Consider a linear dynamical system with inputs $u_t \in \R$ and outputs $y_t \in \R$ evolving according to:
\begin{equation}
    \vec{h}_t = \A \vec{h}_{t-1} + \B u_t \qquad y_t = \C \vec{h}_t,
\end{equation}
where $\vec{h}_t \in \R^{d_{\text{hidden}}}$ is the hidden state. The algorithm we suggest is to predict $y_t$ using the Vovk-Azoury-Warmuth Forecaster with a short autoregressive-with-inputs (ARX) feature vector $[y_{t-1}, \dots, y_{t-n}, u_t, \dots, u_{t-n+1}]$, where $n$ is a hyperparameter that should grow logarithmically with the prediction horizon, see Algorithm~\ref{alg:vaw_sp} for details. By comparison, the original USP algorithm suggests predicting $y_t$ as $$-\sum_{i=1}^n c_i y_{t-i} + \sum_{j=0}^{n-1} \theta_j u_{t-j},$$ 
where the coefficients $c_1, \dots, c_n$ are fixed and come from the Chebyshev polynomial and $\theta_j$ is learned with a first order method.

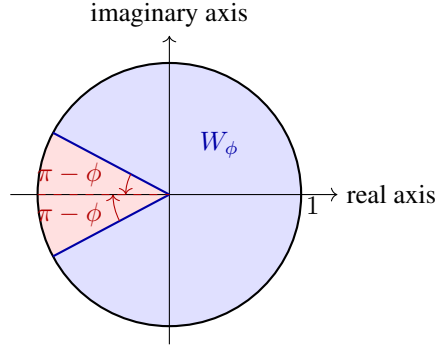
\begin{figure}[h!]
\centering
\begin{tikzpicture}[scale=1.2]
  \def\r{1.45}
  \def\gap{28}
  \fill[blue!12] (0,0) -- (-180+\gap:\r)
    arc[start angle=-180+\gap,end angle=180-\gap,radius=\r] -- cycle;
  \fill[red!12] (0,0) -- (180-\gap:\r)
    arc[start angle=180-\gap,end angle=180+\gap,radius=\r] -- cycle;
  \draw[thick] (0,0) circle (\r);
  \draw[->] (-1.75,0) -- (1.85,0) node[right] {real axis};
  \draw[->] (0,-1.65) -- (0,1.75) node[above] {imaginary axis};
  \draw[blue!65!black,thick] (0,0) -- (180-\gap:\r);
  \draw[blue!65!black,thick] (0,0) -- (-180+\gap:\r);
  \draw[red!70!black,dashed] (-\r,0) -- (0,0);
  \draw[->,red!70!black] (180-\gap:0.48)
    arc[start angle=180-\gap,end angle=180,radius=0.48];
  \draw[->,red!70!black] (180+\gap:0.62)
    arc[start angle=180+\gap,end angle=180,radius=0.62];
  \node[blue!65!black] at (0.55,0.55) {$W_{\phi}$};
  
  \node[red!70!black] at (-1.1,0.22) {$\pi-\phi$};
  \node[red!70!black] at (-1.1,-0.22) {$\pi-\phi$};
  
  \node at (1.58,-0.12) {$1$};
\end{tikzpicture}
\caption{\label{fig:wedge_spectrum} The wedge spectral class
$W_{\phi}=\{z\in\mathbb{C}: |z|\le1,\ |\arg(z)|\le\phi\}$. In our main
result $\phi=\pi-\delta$, so the spectrum may occupy the unit disk except for
a symmetric angular gap of width $2\delta$ around the negative real axis.} 
\end{figure}

At a high level, the effectiveness of a preconditioner, i.e. the set of coefficients $c_1, \dots, c_n$, is governed by two quantities: the effective memory of the transformed representation and the magnitude of the optimal diameter. In the setting of USP, we can think of the preconditioned signal as a compressed representation of the original signal. The compressed representation requires a history which is linear in the degree of the Chebyshev polynomial, however the coefficient magnitude grows very large. This evokes the Minimum Description Length (MDL) principle \cite{rissanen1978modeling, grunwald2007minimum}, which seeks to minimize the combined cost of a model's structural complexity and the bit-length required to represent its parameters. For polynomials, the minimum description length principle suggests that the most efficient representation balances the degree of the polynomial linearly and the coefficients magnitude logarithmically.

Inspired by this principle, we observe that the regret of the Vovk--Azoury--Warmuth (VAW) forecaster \cite{azoury2001relative, vovk2001competitive} matches this balance: it grows linearly with the effective memory and only logarithmically with the the optimal diameter. In order to apply VAW to this setting, we no longer explicitly ``precondition'' the signal $y_t$ by incorporating $\sum_{i=1}^n c_i y_{t-i}$ into our prediction. Instead we include these autoregressive terms into the feature vector. However, the role of USP is still critical because it provides a certificate for an optimal solution and fundamentally explains why the autoregressive terms allow for such an effective memory compression (i.e. that we only need to use the last $n$ inputs and outputs to make good predictions). By moving to the second order VAW algorithm we improve from a cumulative prediction error of $\mathcal{O}(T^{11/13})$ to $\mathcal{O}\!\left(\log^2(T)\right)$. Moreover, USP \cite{marsdenuniversal} only applies to systems whose hidden transition matrix has spectra close to the real line, the complex argument must be bounded by $1/\mathrm{poly} \log (T)$. Our new approach allows us to consider spectra on the entire complex disk.

Indeed, let $W_{\phi} := \left \{z \in \mathbb{C}: \abs{z} \leq 1 \textrm{ and } \arg (z) \leq  \phi \right \}$. Suppose the underlying hidden transition matrix of the data has spectra in $W_{\pi-\delta}$. Then the predictions made by this algorithm achieve cumulative squared loss bounded by $\mathcal{O}\!\left(\delta^{-4}\log^2(TL)\right)$. This is tight in terms of its dependence on $d$. It is known that if the hidden transition matrix is the shift-permutation matrix (which has spectra in $W_{\pi}$) then the regret of any algorithm must grow linearly with hidden dimension \cite{hazan2026spectralfilteringcomplex}. However, we show that this adversarial example truly relies on having spectra in $W_{\pi}$. Indeed if we instead assume the spectra is contained in $W_{\pi - \delta}$ for any constant $\delta > 0$, we achieve dimension-free regret.

In order to show this, we require a new way to guarantee an optimal solution that extends to $W_{\pi-\delta}$, as Chebyshev only applies to $W_{1/ \mathrm{poly} \log (T)}$. The solution comes from Faber polynomials adapted to circular sectors. Our analysis gives a deeper understanding of the role of complex spectra in the difficulty of learning a linear dynamical system. More generally, every useful USP preconditioner is a monic polynomial that is small on the spectrum of the hidden transition matrix. Chebyshev polynomials are optimal for nearly real spectra; Faber polynomials are the corresponding extremal objects for more general compact spectral sets, and in this paper we use them for circular sectors.

\subsection{Our Results}
Our main contributions are as follows:

\paragraph{Logarithmic Regret for Sequence Prediction:} We propose a second-order finite-history predictor: VAW run on the raw feature vector $(y_{t-1},\ldots,y_{t-n},u_t,\ldots,u_{t-n+1})$ (see Algorithm~\ref{alg:vaw_sp}). In Theorem~\ref{thm:poly_preconditioning} we show that if there exists a degree-$n$ monic polynomial with residual size $\alpha_n$ on the spectrum and coefficient scale $G_n$, then there exists a comparator for the raw-feature VAW problem, and VAW suffers only $\mathcal{O}(TL^2\alpha_n^2+Y^2 n\log(G_nYL))$ cumulative prediction loss. For comparison, the original first-order method obtains $\mathcal{O} (n^2 G_n \sqrt{T} + \alpha_n T^2)$ cumulative prediction loss (see Theorem D.1 in \cite{marsdenuniversal}). Instantiating our theorem with the Chebyshev polynomial gives Corollary~\ref{cor:chebyshev}: for spectra in $[-1,1]$, the cumulative prediction loss is $\mathcal{O}\!\left(\log^2(T)\right)$, which is a vast improvement on the $\mathcal{O}(T^{11/13})$ cumulative loss achieved from the first order method.

\paragraph{Extension to Systems with Constant Complex Arguments:} We extend the theoretical guarantees of Universal Sequence Preconditioning (USP) to a strictly broader class of dynamical systems. We use Faber polynomials for circular sectors. Lemma~\ref{lem:faber_wedge_bounds} shows that for every angular gap $\delta>0$ there is a real monic polynomial $p_{n,\delta}$ satisfying $\sup_{z\in W_{\pi-\delta}}|p_{n,\delta}(z)|\le 2e^{-n\delta^2/\pi^2}$ and $\max_i|c_i|\le100^n$. Thus the approximation error remains exponentially small even when the hidden dynamics have constant complex arguments, provided the spectrum avoids a constant wedge around the negative real axis. Note that this contribution is orthogonal to the application of VAW to USP. Indeed, one could apply this result on Faber polynomials to the original first-order method and show that USP applies to a broader class of sequences.

\subsection{Tightness and the Role of the Angular Gap}
Our dimension-free guarantees rely on excluding a constant-width wedge around
the negative real axis. This assumption is not merely an artifact of the proof.
The cyclic-permutation lower bound in Appendix~A of
\cite{hazan2026spectralfilteringcomplex} constructs a $d$-dimensional
marginally stable system for which any accurate finite-memory predictor
requires memory $\Omega(d)$ and therefore accrues cumulative prediction error $\Omega(d)$ for horizon $T \geq d$. Spectrally, this construction is a cyclic
permutation: its eigenvalues are the $d$-th roots of unity. Thus the spectrum
approaches the negative real axis at angular scale $1/d$ and, for even $d$,
contains $-1$ exactly. In the notation of our wedge class, the corresponding
gap parameter is therefore $\delta=\Theta(1/d)$.

This lower bound clarifies the role of Faber and Chebyshev preconditioning.
When the goal is dimension-free prediction, polynomial preconditioners can
compress memory exponentially, provided the spectrum obeys enough angular
separation. If one allows polynomial dependence on the hidden dimension, then
one can always return to the classical Cayley--Hamilton representation, which
uses memory $d$ for arbitrary spectra. The remaining gap is quantitative: our
wedge guarantee scales as $1/\delta^4$ in the final bound, whereas the
permutation lower bound only rules out memory below order $1/\delta$. Closing
this polynomial gap remains open.

\subsection{Related Work}

Our work is most closely related to the literature on prediction in latent linear dynamical systems, rather than recovery. Early spectral-filtering results demonstrated online competitiveness with the best LDS predictor without explicit system identification, but relied on strong spectral structure like symmetry \cite{hazan2017spectral}.  \cite{hazan2018generallds} later extended spectral filtering to general latent LDSs using a convex relaxation that accommodates phases and no longer requires a symmetric transition matrix. However, in the asymmetric setting their guarantee still has polynomial dependence on the horizon and polynomial dependence on the hidden dimension, so it does not yield a polylogarithmic cumulative prediction error bound in the regime we study. Polynomial in $d$ and $T$ bounds are also a classical consequence of the Cayley--Hamilton theorem and realization theory; see, e.g., \cite{kailath1980linear,chen1999linear}.

A second line of work studies \emph{marginally stable} systems more directly. \cite{ghai2020marginallystable} proved no-regret prediction guarantees for marginally stable systems with polynomial system parameter dependence. In the fully observed adversarial setting their bounds remain polynomial in the horizon, while their stronger polylogarithmic-in-\(T\) results rely either on stochastic assumptions or on a reduction to finite-memory autoregressive prediction. These results aren't direct comparators since our focus is absolute cumulative prediction error for a noiseless input--output LDS.

There is also a body of work on partially observed prediction through the Kalman filter. SLIP gives polylogarithmic regret for partially observed LDS prediction under stochastic noise, but its guarantee is stated relative to the Kalman predictor and crucially assumes that the closed-loop matrix \(A-KC\) is diagonalizable with real eigenvalues \cite{rashidinejad2020slip}. Likewise, Tsiamis and Pappas obtain logarithmic regret for online learning of the Kalman filter for nonexplosive systems, including marginally stable ones, again with performance measured against the Kalman predictor \cite{tsiamis2023onlinekalman}. These papers show that polylogarithmic-in-\(T\) prediction is possible under stronger stochastic/filtering assumptions.

The direct asymmetric comparison closest in spirit to our setting is the recent work of Marsden and Hazan, which handles asymmetric marginally stable LDSs and removes hidden-dimension dependence from the regret bound \cite{marsden2025dimensionfree}. However, their horizon dependence remains polynomial, so even after standardizing to the realizable noiseless setting, this line of work falls into the \(O(\mathrm{poly}(T),\mathrm{polylog}(d))\) regime rather than the polylogarithmic-in-\(T\) regime.

Finally, there is a large recovery-oriented literature on learning LDS realizations from input--output data, including semiparametric least squares, stochastic subspace identification, multiscale Hankel/SVD methods, and moment-based recovery \cite{simchowitz2019semiparametric,tsiamis2019stochsysid,lee2022improved,bakshi2023newapproach}. These methods can be used to form predictors. However, their guarantees when converted to cumulative prediction error inevitably lead to a polynomial dependence on the hidden dimension. Related high-dimensional identification results with logarithmic dependence on dimension, such as Fattahi's Markov-parameter learning result, require stronger inherent-stability assumptions and therefore do not directly apply to the marginally stable regime considered here \cite{fattahi2021logsamples}. Of particular interest is the result of \cite{hardt2018gradient}, who curiously have a similar ``Pacman shape" to our wedge $W_\phi$, as the domain in which gradient descent converges to learn the system parameters. In contrast, our domain is where second order methods are able to make an ARX model converge. That said,  we suspect the similarity in domain is more than a coincidence.   

In short, prior work obtains at most two of the following three properties simultaneously: marginal stability, genuinely asymmetric/complex-spectrum hidden dynamics, and polylogarithmic cumulative prediction error with only polylogarithmic dependence on hidden dimension.

\paragraph{Second-Order Online Learning.}
In the realm of online convex optimization, second-order methods are renowned for achieving logarithmic regret for strongly or exponentially concave cost functions. Two of the most prominent algorithms in this space are the Online Newton Step (ONS) \cite{hazan2007logarithmic} and the Vovk--Azoury--Warmuth (VAW) forecaster \cite{vovk2001competitive, azoury2001relative}. However, standard regret guarantees for ONS depend explicitly on the diameter of the decision set. In the USP setting, where the Chebyshev transformation yields a comparator with an exponentially large magnitude, any diameter-dependent bound becomes vacuous. 

Our approach critically relies on the VAW algorithm to circumvent this limitation. A defining property of VAW is that its regret bound is independent of the domain diameter; rather, the magnitude of the comparator enters the bound only logarithmically. This exact property is what allows us to absorb the exponential blowup of the preconditioner's parameters and achieve tight (polylogarithmically) regret rates (or, in this setting, cumulative prediction loss) of $\mathcal{O}(\mathrm{polylog}(T))$, even when accounting for approximation error.

\paragraph{Minimum Description Length (MDL) and Information Theory.}
The connection between online learning, prediction, and data compression is a foundational concept in information theory. The Minimum Description Length (MDL) principle \cite{rissanen1978modeling, grunwald2007minimum} posits that the best statistical model is the one that maximally compresses the data, balancing the complexity of the model itself against the cost of describing the data given the model. The regret of an online learning algorithm is intimately tied to the redundancy in universal coding \cite{cesa2006prediction}. By framing USP through the lens of MDL, we demonstrate that second-order methods achieve an optimal compression scheme for LDS: the effective memory represents the model's structural complexity, while the logarithmic regret represents an optimal bit-length coding of the exponentially large parameter space.

\section{Theoretical Results Overview}
\label{sec:overview}

We first recall the algebraic USP identity, because it is the certificate behind our regret proof. Then we explain how this identity becomes an online learning statement: it exhibits a good comparator for VAW on raw finite-history features. Next we will explain how we achieve state of the art results and tight bounds up to polylogarithmic factors. 

\subsection{Overview of Universal Sequence Preconditioning}
For clarity of presentation, we consider a single-input, single-output Linear Dynamical System where $d_{in} = d_{out} = 1$. At each time step $t$, we observe an input $u_t \in \mathbb{R}$ and an output $y_t \in \mathbb{R}$ according to the law:
\begin{align}
\label{eqn:lds}
\vec{h}_{t} = \vec{A}\vec{h}_{t-1} + \vec{B}u_t \qquad \qquad y_t = \vec{C}\vec{h}_t.
\end{align}
We refer to $\vec{A}$ as the hidden transition matrix. The hidden state $\vec{h}_t$ can be factored out giving
\begin{align*}
    y_t = \sum_{s=0}^t \C \A^{s} \B u_{t-s}.
\end{align*}
For $\theta_s = \C \A^{s} \B$ for $s = 0, \dots, T$, the observation is equivalently $y_t = \sum_{s=0}^t \theta_s u_{t-s}$. Therefore in order to predict a linear dynamical system up to horizon $T$, it suffices to learn $T$ parameters (one for each value of $\C \A^s \B$ for $s = 0, \dots, T$).

The core insight of Universal Sequence Preconditioning \cite{marsdenuniversal} is that the current observation $y_t$ can be represented, up to negligible error, using only $\mathcal{O}(\log(T))$ parameters and history. For history window $n = \lceil \log((T \cdot \norm{\C} \cdot \norm{\B})/ \delta) \rceil$, let $c_0, c_1, \dots, c_n$ be the coefficients of the $n$-th degree monic preconditioning polynomial $p_n(z)=z^n+\sum_{i=1}^n c_i z^{n-i}$, with $c_0=1$. Algebraic manipulation shows that 
\begin{equation}
\label{eqn:lds_breakdown}
    y_t = - \underbrace{\sum_{i=1}^n c_i y_{t-i}}_{\aleph_0} + \underbrace{ \sum_{s = 0}^{n-1}   \left( \sum_{i = 0}^s c_i \C \A^{s-i} \B\right)   u_{t-s}}_{\aleph_1} + \underbrace{ \sum_{s = 0}^{t-n} \C p_n(\A) \A^s \B u_{t-n-s}}_{\aleph_2}.
\end{equation}

Since the monic $n$-th degree Chebyshev polynomial satisfies $\max_{x \in [-1,1]} \abs{p_n^{\text{cheby}}(x)} \leq 2^{-(n-1)}$, if $\A$'s eigenvalues fall in this range, then the term $\aleph_2$ becomes negligible for $n$ large enough. The authors show that you can extend this property to the complex plane with argument bounded by $1/64n^2$ instead of just the interval $[-1,1]$, allowing asymmetric hidden transition matrices. In summary, the seminal work shows that by defining parameter,
\begin{equation}
\label{eqn:theta_usp_def}
    \theta^{p}_s := \sum_{i=0}^s c_i \vec{C} \vec{A}^{s-i} \vec{B}, \text{ with } c_0 = 1,
\end{equation}
then
\begin{equation}
\label{eqn:y_expression}
y_t = -\sum_{i=1}^n c_i y_{t-i} + \sum_{s=0}^{n-1}  \theta^{p}_s u_{t-s} + \epsilon_t,
\end{equation}
where $\abs{\epsilon_t} \leq T \cdot \norm{\C} \cdot \norm{\B} \cdot \kappa \cdot 2^{-n}$, where $\kappa$ is the condition number of the matrix diagonalizing $\vec{A}$, as long as the complex eigenvalues of $\A$ have argument bounded by $1/64n^2$.
More generally, if we remain agnostic to the choice of preconditioning polynomial, then Eq.~\ref{eqn:y_expression} holds with $\abs{\epsilon_t} \leq T\cdot \norm{\C}\cdot \norm{\B}\cdot \kappa\cdot \sup_{z\in K}|p_n(z)|$ whenever $K$ contains $\lambda(\A)$. We will show that if we consider $K$ as the wedge of the complex disk then the Faber polynomial minimizes $\sup_{z\in K}|p_n(z)|$.

\subsection{Contribution One: The Power of a Second Order Method}

Equation~\ref{eqn:y_expression} suggests a simple learning algorithm--- learn $\theta^{p}$ via regression. Perhaps even simpler and cleaner is to learn the entire vector $(-c_1, \dots, -c_n, \theta_0^p, \dots, \theta_{n-1}^p)$. A naive application of first order learning methods gives regret which scales linearly with the norm of this parameter. This is unfortunate since this norm grows exponentially with $n$; from Lemma 3.2 of \cite{marsdenuniversal}, $\max_{i = 0, \dots, n} \abs{c_i}$ can grow as large as $2^{0.3n}$ for Chebyshev preconditioning, while Lemma~\ref{lem:faber_wedge_bounds} gives the analogous bound $100^n$ for the sector Faber preconditioner. Although $\theta^{p}$ can have a tiny dimension which is only logarithmic in $T$, its norm does not reflect the reduction in domain. This is where the insight of our work comes in. Inspired by the minimum description length principle, we consider applying a second order learning method to sequence preconditioning. Critically, second order learning methods do not treat the \emph{dimension} and the \emph{norm} of the learnable parameter space as equal citizens in the way that first order methods do. 

The error of the preconditioning method has two components: one from approximation and the other from learning/estimation. While Chebyshev and Faber polynomials yield optimal exponential decay for approximation error, the learning error---previously relying on first-order methods---can be significantly improved. Specifically, the original USP algorithm uses online gradient descent to learn the parameters of the model which grow exponentially with the degree of the polynomial. We suspect that this tradeoff is necessary, formalized in the following conjecture.

\begin{conjecture}
\label{conjecture:coeff_growth}
Let $p_n(x) = x^n + \sum_{j=0}^{n-1} c_j x^j$ denote any monic polynomial. Let $\mathcal{D} \subseteq \mathbb{C}$ denote any convex and compact set with area $D_A>0$. If there exist global constants $C, c>0$ such that for all $n \geq 1$,
$\max_{x \in \mathcal{D}} \abs{p_n(x)} \leq C \cdot 2^{-cn}$ then there must also exist constants $B, b >0$ such that for all $n \geq 1$, $\max_{j \in [n-1]} \abs{c_j} \geq  B 2^{b n}$.
\end{conjecture}

In Appendix~\ref{appendix:coeff_growth} we prove a weaker statement than the above in which we assume $c = 1$. Regardless of the truth of Conjecture~\ref{conjecture:coeff_growth}, we observe that by using the VAW algorithm we can obtain a regret bound which grows logarithmically with these large coefficients, unlocking tight (up to logarithms) regret bounds.

\subsection{Contribution Two: Extending to Asymmetric Systems Via Faber Polynomials}

In the original USP work, the results only apply to linear dynamical systems whose hidden transition matrix $\A$ has eigenvalues in the region $\left \{ z \in \mathbb{C} \textrm{s.t.} \abs{z} \leq 1 \textrm{ and } \arg(z) \leq 1/ \mathrm{poly} \log (T) \right \}$. There is intuition for why the spectrum of $\A$ must be restricted. Since the resulting regret is hidden-dimension independent, the spectrum of $\A$ must be restricted. Indeed, if it were not then it would contradict a lower bound. The lower bound comes from considering $\A$ to be the shift permutation matrix.For this cyclic-permutation example the spectrum consists of roots of unity; it approaches the negative real axis at angular scale $1/d$ and contains $-1$ when $d$ is even, and any accurate finite-memory predictor requires memory $\Omega(d)$ \cite{hazan2026spectralfilteringcomplex}. If polynomial dependence on $d$ is allowed, the classical Cayley--Hamilton representation gives a degree-$d$ autoregressive predictor for arbitrary spectra. Our goal is different: to understand when dimension-free prediction is possible. Indeed, in the regime where $\A$ is not adversarially constructed, the question remains whether it is possible to do better with respect to hidden dimension. What is the exact trade off as $\arg(z)$ approaches $\pi$ and the resulting regret?

We begin answering this question by extending the comparator-existence argument behind USP to a strictly broader class of dynamical systems. For wedge spectra, the proof uses a sector Faber polynomial to show that a good short-memory comparator exists even when the eigenvalues have constant complex argument, see Corollary~\ref{cor:main}. For spectra in $W_{\pi-\delta}$, Lemma~\ref{lem:faber_wedge_bounds} gives
\[
\sup_{z\in W_{\pi-\delta}}|p_{n,\delta}^{\mathrm{Faber}}(z)|\le 2e^{-n\delta^2/\pi^2},
\qquad
\max_i |c_i|\le 100^n.
\]
The residual therefore decays exponentially in $n\delta^2$, while VAW pays only logarithmically for the exponential coefficient growth. This yields memory $n=\tilde O(\delta^{-2})$ and the cumulative loss bound $\tilde O(Y^2\delta^{-4})$, interpolating between dimension-free prediction for constant angular gaps and the Cayley--Hamilton regime when the gap shrinks with dimension. Moreoever, whenever $\delta = \Omega(d^{-1/4})$ the resulting regret is sublinear in $d$, which the Cayley--Hamilton regime would not be able to achieve.

The shape in Figure~\ref{fig:wedge_spectrum} is important because it is the spectral regime in which the Faber certificate replaces the Chebyshev certificate. Prior USP-style analyses required the angular component of the eigenvalues to shrink with the horizon. By contrast, if the eigenvalues lie in $W_{\pi-\delta}$ for any constant $\delta>0$, the preconditioner remains uniformly small on the whole spectrum while its coefficients grow only
exponentially with the degree. This is exactly the tradeoff that VAW can absorb,
and it is what yields the $\pi-\Omega(1)$ angle entry in Table~\ref{table:shalom}. Note that the algorithm run in all cases is the same raw-feature VAW forecaster.

\section{Main Theorems and Proofs}

Recall the standard Vovk-Azoury-Warmuth (VAW) Algorithm \cite{azoury2001relative, vovk2001competitive}, restated in Algorithm~\ref{alg:vaw} for the reader's convenience. Theorem~\ref{thm:vaw} states the standard regret bound achieved by VAW. We assume the regularization parameter satisfies $\lambda \ge 1/\|\vec{x}^*\|_2^2$.

\begin{algorithm}
\caption{Vovk-Azoury-Warmuth (VAW) Forecaster \cite{azoury2001relative, vovk2001competitive}}
\label{alg:vaw}
\begin{algorithmic}[1]
\State \textbf{Input:} regularization $\lambda \ge 1/\|\vec{x}^*\|_2^2$
\State \textbf{Initialize:} $\A_0 = \lambda I_d$, target vector $\vec{v}_0 = 0 \in \mathbb{R}^d$
\For{$t = 1$ to $T$}
    \State Observe feature vector $\vec{a}_t$
    \State Update curvature matrix: $\A_t = \A_{t-1} + \vec{a}_t \vec{a}_t^\top$
    \State Predict $\hat{y}_t = \vec{a}_t^\top \A_t^{-1} \vec{v}_{t-1}$
    \State Observe true label $y_t$ and suffer loss $l_t(\hat{b}_t) = \frac{1}{2}(\hat{b}_t - b_t)^2$
    \State Update target vector: $\vec{v}_t = \vec{v}_{t-1} + b_t \vec{a}_t$
\EndFor
\end{algorithmic}
\end{algorithm}

\begin{theorem}[VAW Regret]
\label{thm:vaw}
Assume the true labels are uniformly bounded by $Y$, i.e. $|y_t| \le Y$ for all $t=1, \dots, T$. Assume the feature vectors are uniformly bounded by $R$, i.e. $\norm{\vec{a}_t}_2 \leq R$ for all $t=1, \dots T$. For any benchmark $\vec{x}^* \in \mathbb{R}^n$, the regret is bounded by:
\begin{equation}
\label{eqn:vaw_regret}
\sum_{t=1}^T \frac{1}{2}(\hat{y}_t - y_t)^2 \le \sum_{t=1}^T \frac{1}{2}({\vec{x}^*}^\top \vec{a}_t - y_t)^2 + \frac{1}{2} + \frac{Y^2}{2} n \log\left(1 + \frac{T \|\vec{x}^*\|_2^2 R^2}{n}\right) 
\end{equation}
\end{theorem}

\paragraph{Transforming sequence prediction with short memory into linear prediction with small dimension:} In Algorithm~\ref{alg:vaw_sp} we show exactly how to convert the sequence prediction problem into the VAW framework. Using Equation~\ref{eqn:y_expression}, we frame predicting $y_t$ as a linear prediction task with dimension $2n$. Let $\vec{a}_t$ be the vector of the most recent $n$ observations and the most recent $n$ inputs:
\begin{equation}
\vec{a}_t \leftarrow [y_{t-1}, \dots, y_{t-n}, u_{t}, \dots, u_{t-n+1}].
\end{equation}

Recall the definition of $\theta_s^{p}$ from Eq.~\ref{eqn:theta_usp_def} and recall that $c_1, \dots, c_n$ represent the coefficients of the $n$-th degree monic preconditioning polynomial.
We define our target weight vector $\vec{x}^*$ as the concatenation of the polynomial coefficients and $\theta_s^{p}$, 
\begin{align*}
\vec{x}^* &\leftarrow [-c_1, \dots, -c_n, \theta_0^{p}, \dots, \theta_{n-1}^{p}]
\end{align*}

Using our definitions of $\vec{x}^*$ and $\vec{a}_t$ we have:
\begin{equation}
\label{eqn:y_t_lin}
y_t = {\vec{x}^*}^\top \vec{a}_t + \epsilon_t. 
\end{equation}

\begin{algorithm}
\caption{Vovk-Azoury-Warmuth (VAW) Forecaster for Sequence Preconditioning}
\label{alg:vaw_sp}
\begin{algorithmic}[1]
\State \textbf{Input:} regularization $\lambda > 0$
\State \textbf{Initialize:} $\A_0 = \lambda I_d$, target vector $\vec{v}_0 = 0 \in \mathbb{R}^d$
\For{$t = 1$ to $T$}
    \State Observe input $u_t$
    \State Create feature vector $\vec{a}_t \gets [y_{t-1}, \dots, y_{t-n}, u_{t}, \dots, u_{t-n+1}]$
    \State Update curvature matrix: $\A_t = \A_{t-1} + \vec{a}_t \vec{a}_t^\top$
    \State Predict $\hat{y}_t = \vec{a}_t^\top \A_t^{-1} \vec{v}_{t-1}$
    \State Observe true label $y_t$ and suffer loss $l_t(\hat{y}_t) = \frac{1}{2}(\hat{y}_t - y_t)^2$
    \State Update target vector: $\vec{v}_t = \vec{v}_{t-1} + y_t \vec{a}_t$
\EndFor
\end{algorithmic}
\end{algorithm}

The critical property of the VAW algorithm is that its regret compared to any benchmark $\vec{x}^*$ scales linearly with the \emph{dimension} of the target and logarithmically with the \emph{norm} of the benchmark vector. With our mapping above, the dimension of the target corresponds exactly to the effective memory of the signal.  This is nicely compatible with Universal Sequence Preconditioning since it guarantees the existence of a benchmark $\vec{x}^*$, as above, which has negligible prediction error and very small dimension, i.e. ($2n$ is $\mathcal{O}(\log(T))$), but unfortunately large norm.

\begin{theorem}[VAW with a polynomial comparator certificate]
\label{thm:poly_preconditioning}
Suppose $y_1,\dots,y_T$ comes from Eq.~\ref{eqn:lds}. Assume
$\A=\vec{P}\Lambda\vec{P}^{-1}$, $\max_t|y_t|\le Y$ for some $Y\ge1$, and
$\max_t|u_t|\le1$.
Let $K\subseteq\{z:|z|\le1\}$ contain $\lambda(\A)$, and let
\[
  p_n(z)=z^n+\sum_{i=1}^n c_i z^{n-i}
\]
be a monic preconditioning polynomial such that
\[
  \sup_{z\in K}|p_n(z)|\le \alpha_n,
  \qquad
  \max_{0\le i\le n} |c_i|\le G_n,
\]
where $c_0=1$ and $G_n\ge1$. Denote
\[
  \kappa=\norm{\vec{P}}_2\norm{\vec{P}^{-1}}_2,
  \qquad
  L=1+T(1+\kappa\norm{\C}_2\norm{\B}_2).
\]
Then Algorithm~\ref{alg:vaw_sp}, run with memory length $n$, satisfies
\[
\sum_{t=1}^T(\hat y_t-y_t)^2
\le
O\!\left(
  TL^2\alpha_n^2
  +
  Y^2 n\log\!\left(G_nYL\right)
\right).
\]
\end{theorem}

\begin{proof}
By Eq.~\ref{eqn:y_expression},
\[
  y_t
  =
  -\sum_{i=1}^n c_i y_{t-i}
  +\sum_{s=0}^{n-1}\theta_s^p u_{t-s}
  +\epsilon_t,
\]
with $|\epsilon_t|\le L\alpha_n$. Define the comparator
\[
  \vec{x}^*=[-c_1,\ldots,-c_n,\theta_0^p,\ldots,\theta_{n-1}^p]
\]
for the feature vector
$\vec{a}_t=[y_{t-1},\ldots,y_{t-n},u_t,\ldots,u_{t-n+1}]$. Then
${\vec{x}^*}^{\top}\vec{a}_t=y_t-\epsilon_t$, so the comparator's total
approximation loss is at most
\[
  \sum_{t=1}^T
  \left({\vec{x}^*}^{\top}\vec{a}_t-y_t\right)^2
  \le
  TL^2\alpha_n^2.
\]
For each $s<n$,
\[
  |\theta_s^p|
  =
  \left|\sum_{i=0}^s c_i\C\A^{s-i}\B\right|
  \le (s+1)G_nL\le nG_nL,
\]
so
\[
  \|\vec{x}^*\|_2
  \le O(n^{3/2}G_nL).
\]
The feature vectors have dimension $2n$ and norm at most
$R=\sqrt{n(Y^2+1)}\le\sqrt{2n}Y$. Substituting these bounds into
Theorem~\ref{thm:vaw}, and absorbing polynomial factors in $T,n,L$ into the
logarithm, gives
\[
\sum_{t=1}^T(\hat y_t-y_t)^2
\le
O\!\left(
  TL^2\alpha_n^2
  +
  Y^2 n\log(G_nYL)
\right),
\]
as claimed.
\end{proof}

First we show how the regret substantially improves upon the original first-order method from \cite{marsdenuniversal} in the original setting where the preconditioning polynomial is the Chebyshev polynomial. 

\begin{corollary}[Chebyshev instantiation for real spectra]
\label{cor:chebyshev}
Suppose $y_1,\dots,y_T$ comes from Eq.~\ref{eqn:lds}, where
$\A=\vec{P}\Lambda\vec{P}^{-1}$ and $\lambda(\A)\subseteq [-1,1]$. Assume $\max_t |y_t|\le Y$ for some $Y\ge1$ and $\max_t |u_t|\le 1$. Let
$n=\left\lceil \log_2(4TL^2)\right\rceil$.
Run Algorithm~\ref{alg:vaw_sp} with memory length $n$ on the raw sequence
features
$
\vec{a}*t =
[y*{t-1},\ldots,y_{t-n},u_t,\ldots,u_{t-n+1}].
$
Then
\[ 
\sum_{t=1}^T(\hat y_t-y_t)^2
\le
\mathcal{O} \left(
Y^2 n\left[n+\log(YL)\right]
\right) = 
\mathcal{O} \left(
Y^2 \log^2(TYL)
\right).
\]
In particular, for real marginally stable spectra, second-order sequence prediction achieves polylogarithmic cumulative prediction error, with no polynomial dependence on the hidden dimension except through
$\kappa, \norm{\C}_2$, and $\norm{\B}_2$.
\end{corollary}

\begin{proof}
We apply Theorem~\ref{thm:vaw} with the degree-$n$ monic
Chebyshev polynomial $M_n(z) = 
z^n+\sum_{i=1}^n c_i z^{n-i}$.
(The algorithm does not need to know these coefficients. They are used only to exhibit a comparator for the VAW regret bound.) The classical Chebyshev extremal property gives $\sup_{z\in[-1,1]} |M_n(z)|\le 2^{1-n}$. Moreover, the monomial coefficients of $M_n$ are exponentially bounded: for a
universal constant $C_{\mathrm{cheb}}>1$, $
\max_{0\le i\le n}|c_i|\le C_{\mathrm{cheb}}^n$. 
Thus Theorem~\ref{thm:vaw} applies with 
\[ 
\alpha_n=2^{1-n},
\qquad
G_n=C_{\mathrm{cheb}}^n .
\]
The approximation term satisfies $TL^2\alpha_n^2 = TL^2 2^{2-2n} \le 1$ 
by the choice of $n$. The learning term is
$ Y^2 n\log(G_nYL)= \mathcal{O} \left(
Y^2 n\left[n+\log(YL)\right]
\right)$. 
Since $n=\left\lceil \log_2(4TL^2)\right\rceil = 
\mathcal{O}(\log(TYL))$, 
we obtain
\[ 
\sum_{t=1}^T(\hat y_t-y_t)^2 = \mathcal{O} \left(
Y^2\log^2(TYL)
\right),
\]
as claimed.
\end{proof}

\begin{lemma}[Faber polynomial bounds for wedge spectra]
\label{lem:faber_wedge_bounds}
For every $\delta\in(0,\pi)$ and $n\ge1$, there is a real monic polynomial
\[
  p_{n,\delta}(z)=z^n+\sum_{i=1}^n c_i z^{n-i}
\]
with $c_0:=1$ such that
\[
  \sup_{z\in W_{\pi-\delta}}|p_{n,\delta}(z)|
  \le
  2e^{-n\delta^2/\pi^2},
  \qquad
  \max_{0\le i\le n}|c_i|\le 100^n .
\]
\end{lemma}

\begin{corollary}[Main result: VAW with Faber preconditioning]
\label{cor:main}
Suppose $y_1,\dots,y_T$ comes from Eq.~\ref{eqn:lds}, where
$\A=\vec{P}\Lambda\vec{P}^{-1}$ and
$\lambda(\A)\subseteq W_{\pi-\delta}$ for some $\delta\in(0,\pi)$.
Assume $\max_t|y_t|\le Y$ for some $Y\ge1$ and $\max_t|u_t|\le1$. Set
\[
  n=
  \left\lceil
    \frac{\pi^2}{\delta^2}\log(4T^2L)
  \right\rceil .
\]
Then Algorithm~\ref{alg:vaw_sp}, run with memory length $n$, satisfies
\[
\sum_{t=1}^T(\hat y_t-y_t)^2
=
O\!\left(
  Y^2 n
  \left[
    n+
    \log(YL)
  \right]
\right)
=
O\!\left(
  \frac{Y^2}{\delta^4}
  \log^2(TYL)
\right).
\]
In particular, for fixed $\delta,Y,\kappa,\norm{\C}_2$, and $\norm{\B}_2$, the
bound is polylogarithmic in $T$. The hidden dimension enters only through
$\kappa,\norm{\C}_2$, and $\norm{\B}_2$.
\end{corollary}

\begin{proof}
Apply Lemma~\ref{lem:faber_wedge_bounds} and Theorem~\ref{thm:poly_preconditioning}
with $K=W_{\pi-\delta}$,
\[
  \alpha_n=2e^{-n\delta^2/\pi^2},
  \qquad
  G_n=100^n .
\]
The approximation term is at most a constant for the stated choice of $n$, since
$TL^2\alpha_n^2\le 1/T$. Also $\log(G_n)=O(n)$. Theorem~\ref{thm:poly_preconditioning}
therefore gives
\[
\sum_{t=1}^T(\hat y_t-y_t)^2
  \le
O\!\left(
  Y^2 n
  \left[
    n+
    \log(YL)
  \right]
\right).
\]
Since $Y\ge1$ and $L\ge1$, $\log(4T^2L)=O(\log(TYL))$ and
$\log(YL)\le \log(TYL)$. Also, because $\delta\in(0,\pi)$, the ceiling in the
definition of $n$ gives $n=O(\delta^{-2}\log(TYL))$. Substituting this value of
$n$ gives the expanded bound.
\end{proof}

\section{Experiments}
\label{sec:experiments}

\paragraph{Data Generation.}
We evaluate Universal Sequence Preconditioning (USP) on synthetic data generated from a Linear Dynamical System (LDS) governed by the state-space equations from Eq.~\ref{eqn:lds}. We introduce additive observation noise $v_t \sim \mathcal{N}(0, \sigma^2)$ with $\sigma = 0.01$. To test the constant-angle regime captured by Corollary~\ref{cor:main}, we specify the eigenvalues of the transition matrix $\vec{A}$ to have a magnitude of $1 - \delta$ with $\delta = 10^{-7}$ (determining the spectral gap) and imaginary components bounded by a threshold $\tau = 0.1$. The input and output matrices, $\vec{B}$ and $\vec{C}$, are sampled from a standard normal distribution and normalized such that $\|\vec{B}\|_2 = \|\vec{C}\|_2 = 1$. We evaluate two primary configurations:
\begin{enumerate}
\item \textbf{High-dimensional}: $d_{\text{hidden}} = 100$ and baseline feature window size $w = 100$.
\item \textbf{Low-dimensional}: $d_{\text{hidden}} = 10$ and baseline feature window size $w = 10$.
The parameter $w$ dictates the history length for the un-preconditioned baseline models.
\end{enumerate}
Both configurations use a sequence length of $T = 5000$. Inputs $u_t$ are generated by sampling from $\mathcal{N}(0,1)$ and scaling by $1/\log(5(t+2))$ to simulate a decaying excitation signal. To evaluate robustness under poor conditioning, we also test highly correlated input sequences generated from an AR(1) process with coefficient $\rho = 0.99$.

\paragraph{Performance Measurement.}
We compare Online Gradient Descent (OGD), Adam, and the Vovk-Azoury-Warmuth (VAW) forecaster. For each method, we perform a grid search over hyperparameters and report results using the optimal configuration. The grids are defined as follows:
\begin{itemize}
    \item OGD and Adam learning rates: $\eta \in \{10^{-4}, 10^{-3}, 0.01, 0.05, 0.1\}$.
    \item VAW regularization parameter: $\lambda \in \{10^{-2}, 10^{-1}, 1, 10, 100, 1000\}$.
\end{itemize}
The preconditioning polynomial degree $n$ is swept from $0$ to $30$. We evaluate two approaches for the polynomial:
\begin{enumerate}
\item \textbf{Chebyshev}: Fixed coefficients $c_1, \dots, c_n$ derived from monic Chebyshev polynomials. The VAW targets the preconditioned signal directly by using $\tilde{y}_t = y_t + \sum_{i=1}^n c_i y_{t-i}$ as labels and $u$ as features.
    \item \textbf{Learnable}: Polynomial coefficients treated as learnable parameters optimized alongside the model. This is exactly Algorithm~\ref{alg:vaw_sp}.
\end{enumerate}

Performance is quantified by the normalized prediction error,
$\frac{|y_t - \hat{y}_t|}{|y_t| + \epsilon}$ (with $\epsilon = 10^{-8}$). To capture steady-state performance, we compute the mean error over the final $200$ steps of the sequence for each trial, and report the median of this metric across $10$ independent trials. Error bars indicate the interquartile range (25th--75th percentiles).

We evaluate the performance of OGD, Adam, and VAW across different degrees of preconditioning. Figure~\ref{fig:experimental_results} presents the results for both high-dimensional ($d=100$) and low-dimensional ($d=10$) regimes. As anticipated by our theoretical analysis, the Chebyshev-preconditioned VAW achieves superior steady-state error reduction, particularly in the high-dimensional setting where learning the annihilating polynomial is prohibitively complex.

\begin{figure}[htbp]
    \centering
    \begin{subfigure}[b]{0.48\textwidth}
        \centering
        \includegraphics[width=\textwidth]{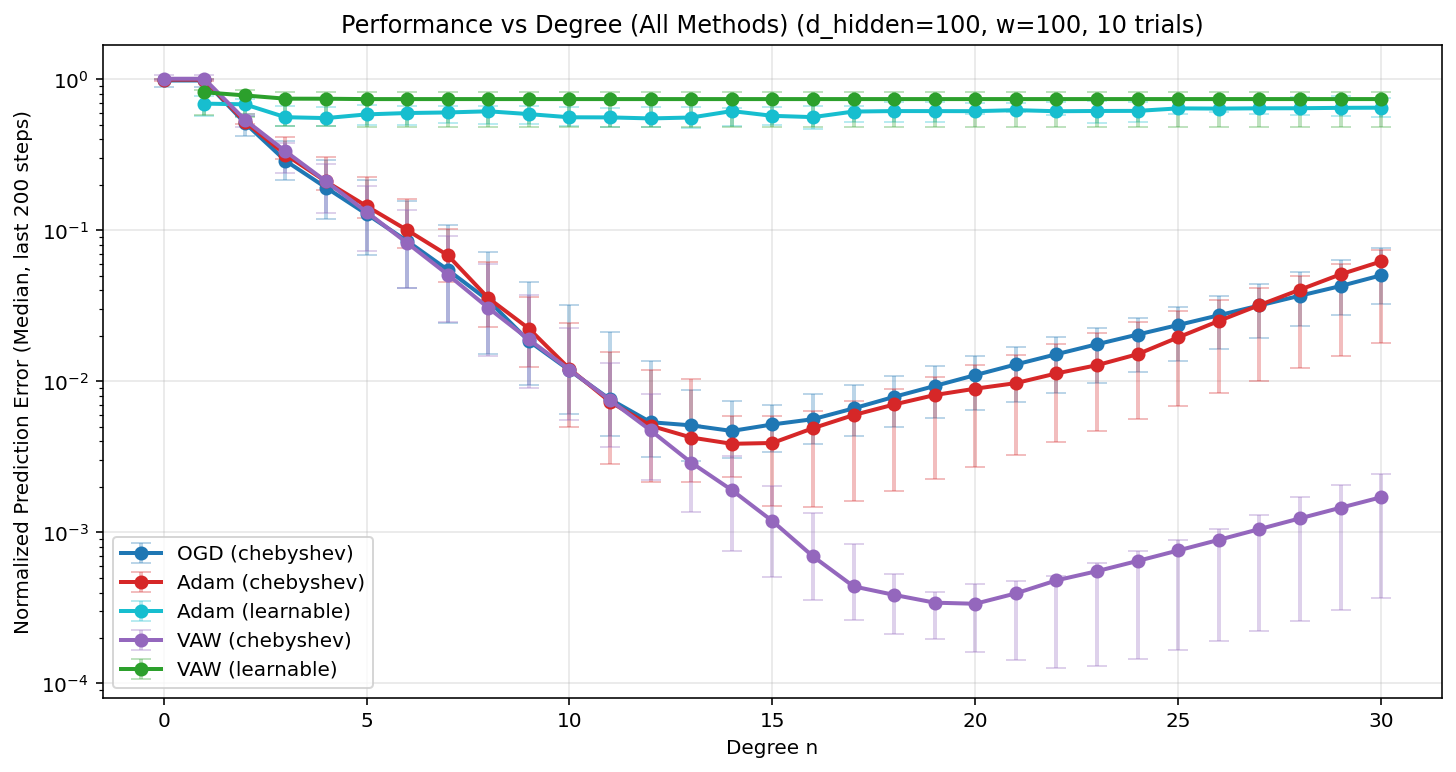}
        \caption{Hidden dimension $d = 100$.}
        \label{fig:error_d100}
    \end{subfigure}
    \hfill
    \begin{subfigure}[b]{0.48\textwidth}
        \centering
        \includegraphics[width=\textwidth]{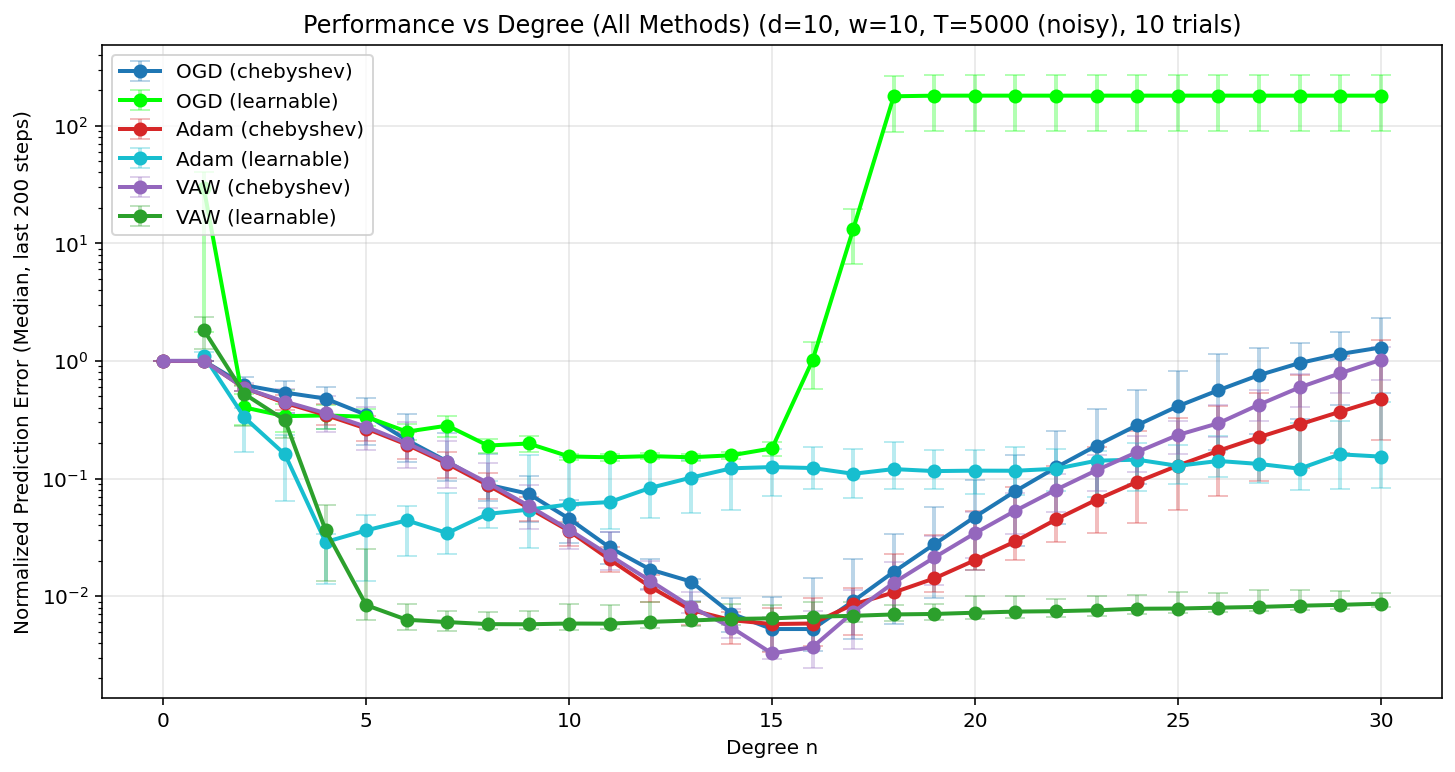}
        \caption{Hidden dimension $d = 10$.}
        \label{fig:error_d10}
    \end{subfigure}
    \caption{Normalized prediction error versus degree of preconditioning polynomial for (a) $d=100$ and (b) $d=10$. The Chebyshev-preconditioned VAW achieves superior steady-state error reduction compared to OGD and Adam as the degree increases.}
    \label{fig:experimental_results}
\end{figure}

\paragraph{Results Discussion.}
We observe the impact of hidden dimension on the importance of preconditioning with Chebyshev coefficients over learning them. By the Cayley-Hamilton theorem, a perfect autoregressive predictor of degree $d_{\text{hidden}}$ theoretically exists for this system. As we see in Figure~\ref{fig:experimental_results}, when the hidden dimension is high enough, adaptively learning $2n$ parameters (which are functionally approximating the $d_{\text{hidden}}$-dimensional system) fails to converge effectively for OGD and Adam. However, fixed Chebyshev preconditioning of degree $n = \mathcal{O}(\log T)$ elegantly bypasses this parameter learning phase.

Next, these results are consistent with the main hypothesis of our paper: VAW is much more stable than first-order methods
(i.e. OGD and Adam) to the large coefficients needed to optimally precondition the signal. Indeed, after degree 14 the
first-order methods stop gaining benefits from preconditioning and their performance worsens. However, the VAW method continues to leverage higher degrees, getting optimal prediction error at degree 20 before its performance worsens.

\subsection{Signal Norm and Preconditioning Dynamics}
\label{subsec:signal_norm}
Our theoretical results are impeded by the inability to prove that the preconditioned target signal will have small norm. If we could bound the norm of the preconditioned signal that didn't grow with the size of the preconditioning coefficients then two important changes could be made to our results: 1) we would not need to learn the autoregressive coefficients and 2) we could just as easily apply Online Newton Step (ONS) to get the state of the art regret bounds. Empirically, however, we see that the preconditioned target signal actually has much smaller norm than the raw signal in many regimes. Figures~\ref{fig:raw_vs_cheby_d12} and \ref{fig:cheby_vs_diff_d12} illustrate the raw signal, generated synthetically as described above, alongside the Chebyshev preconditioned signal (with degree 12) as well as the differenced signal. Figure~\ref{fig:norm_sweep} sweeps over 100 random trials and plots the maximum norm of the signal along the horizon. We plot the median of the trials and give a shaded region to represent the interquartile range (25th--75th percentiles). 

\begin{figure}[htbp]
    \centering
    \begin{subfigure}[b]{0.32\textwidth}
        \centering
        \includegraphics[width=\textwidth]{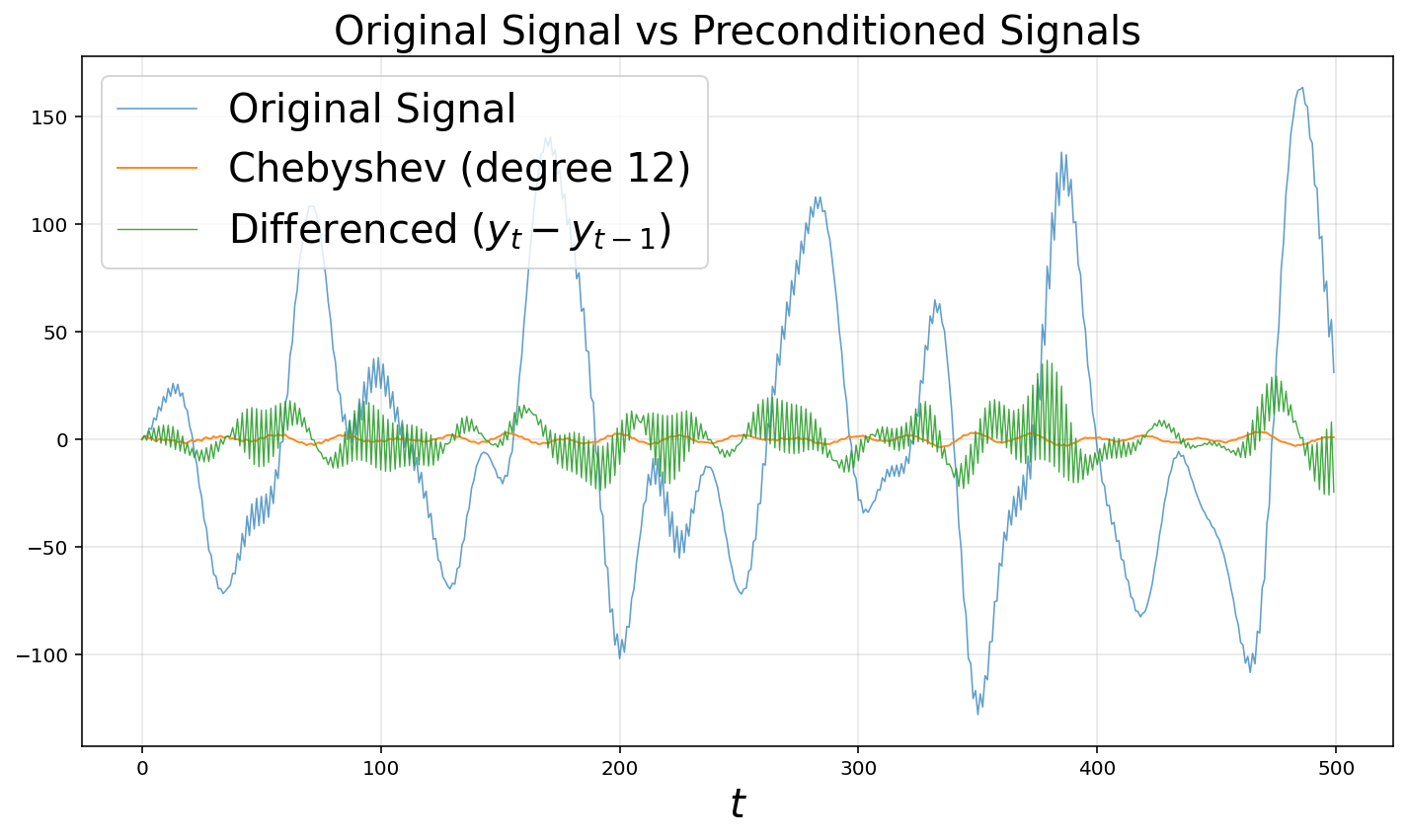}
        \caption{Raw versus preconditioned signals.}
        \label{fig:raw_vs_cheby_d12}
    \end{subfigure}
    \hfill
    \begin{subfigure}[b]{0.32\textwidth}
        \centering
        \includegraphics[width=\textwidth]{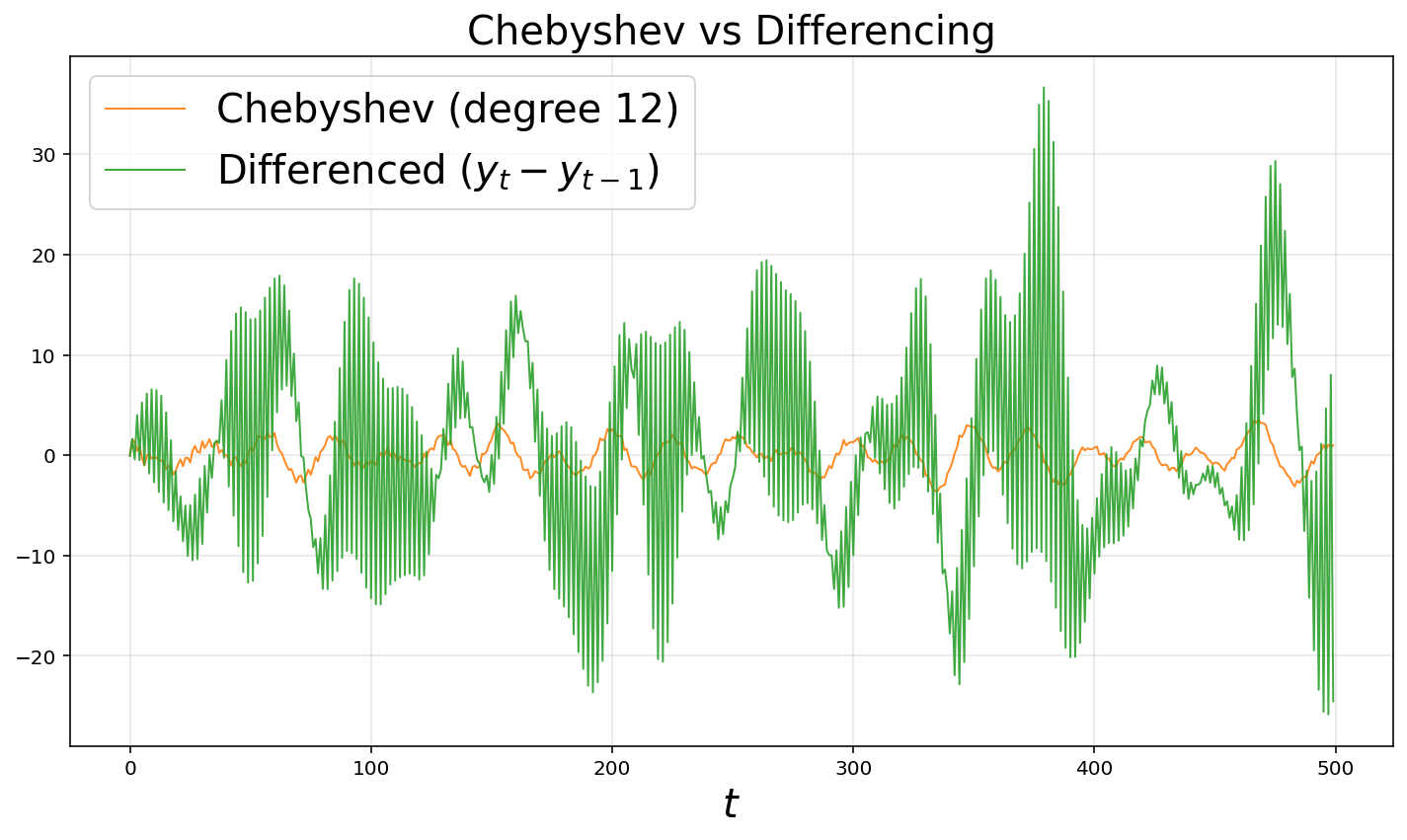}
        \caption{Chebyshev preconditioning versus differencing.}
        \label{fig:cheby_vs_diff_d12}
    \end{subfigure}
    \hfill
    \begin{subfigure}[b]{0.32\textwidth}
        \centering
        \includegraphics[width=\textwidth]{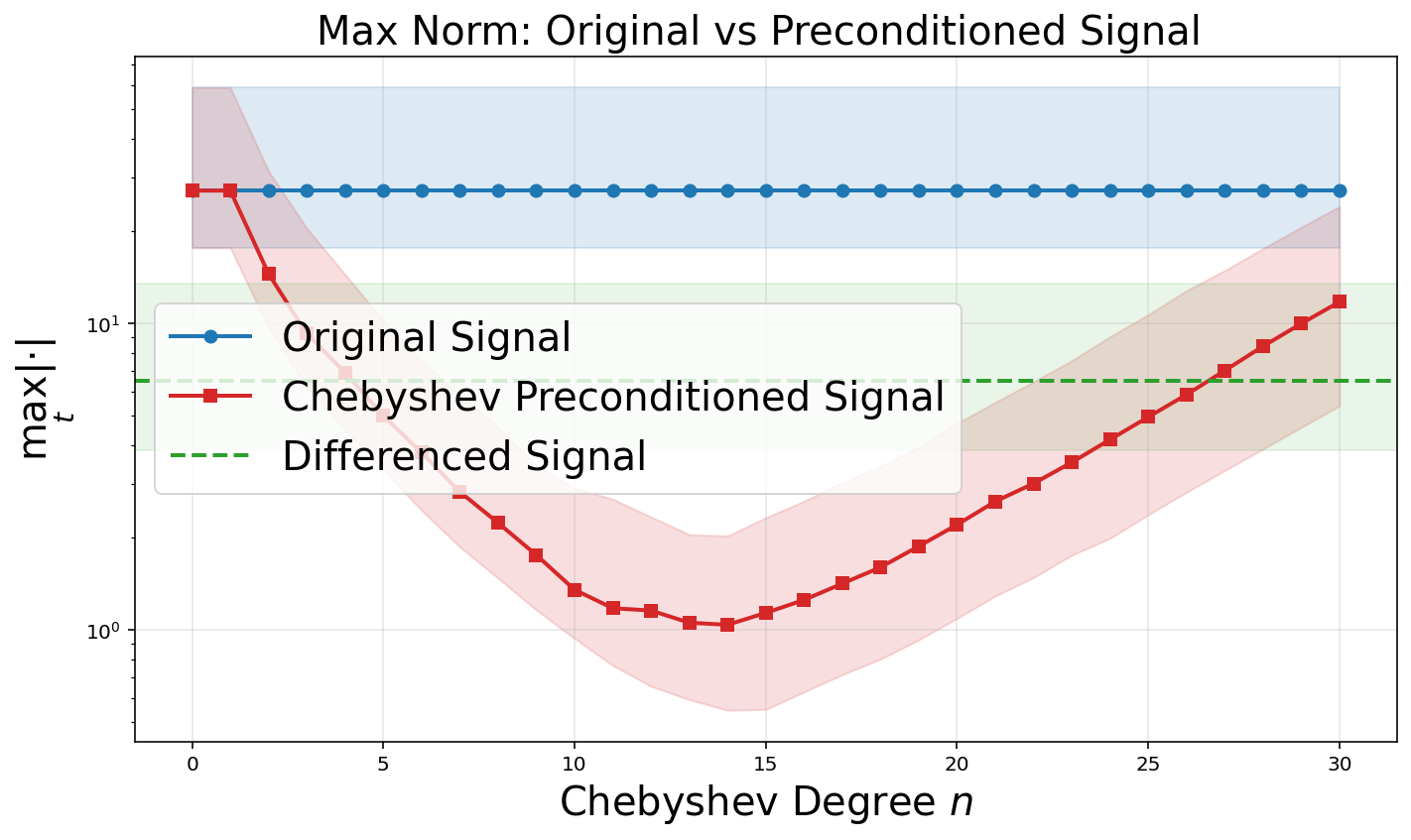}
        \caption{Maximum norm versus Chebyshev degree.}
        \label{fig:norm_sweep}
    \end{subfigure}
    \caption{Preconditioning empirically reduces the norm of the signal. If this is provable, learning the autoregressive coefficients for Algorithm~\ref{alg:vaw_sp} is unnecessary.}
    \label{fig:signal_plots}
\end{figure}

\section{Conclusion}
\paragraph{Summary} We addressed a fundamental tension in sequence prediction for marginally stable linear dynamical systems: the trade-off between effective memory compression and the exponential explosion of parameter magnitudes. By interpreting Universal Sequence Preconditioning (USP) through the lens of the Minimum Description Length (MDL) principle, we established that second-order online learning methods---specifically the Vovk-Azoury-Warmuth (VAW) forecaster---are uniquely equipped to absorb the parameter blowup induced by polynomial preconditioning. Our theoretical framework yields near-optimal guarantees. By shifting the learning burden from the parameter norm to the effective dimension, we exponentially improved the cumulative prediction regret from the $\tilde{\mathcal{O}}(T^{11/13})$ bounds of prior first-order approaches to $\mathcal{O}(\text{polylog}(T))$, completely eliminating polynomial dependencies on both the time horizon and the hidden state dimension. Furthermore, we introduce Faber preconditioning to extend the theoretical viability of USP to a strictly broader class of dynamical systems with constant complex arguments, namely spectra in $W_{\pi-\delta}$ for constant $\delta>0$. Empirically, our evaluations robustly corroborate these theoretical insights. The Chebyshev-preconditioned VAW algorithm consistently outperforms both first-order methods and adaptive-polynomial approaches, particularly in adverse regimes characterized by large time horizons, high-dimensional hidden states, and highly correlated input sequences. 

\paragraph{Limitations} 
Our theoretically analyzed VAW algorithm requires learning the autoregressive coefficients. This is a slight departure from the spirit of USP since an offline transformation to the target sequence makes it easier to learn whereas in our case we must learn this transformation. However, it is also deeply connected to USP since its proof heavily relies on construction of the benchmark $\vec{x}^*$, which is the polynomial preconditioned solution.
Moreover, empirical evidence suggests this requirement may be an artifact of our proof technique. Another limitation of this work is the practicality of the VAW algorithm. It requires a much larger memory footprint than the first order methods. The memory footprint grows poly-logarithmically with horizon length and as $\tilde O(\delta^{-2})$ in the Faber wedge parameter.

\paragraph{Future Directions} Promising directions for future work include extending these second-order preconditioning frameworks to non-linear dynamical systems, perhaps by using the recent breakthrough of \cite{dogariu2025universal}. Another important direction is to sharpen the dependence on the angular gap $\delta$, since the current Faber proof gives a $\delta^{-4}$ loss dependence while the cyclic-permutation lower bound only forces memory on the order of $1/\delta$. Ultimately, our results demonstrate that when appropriately paired with second-order optimization, universal preconditioning unlocks fundamentally new capabilities for scalable, highly accurate sequence prediction.

\newpage

\bibliographystyle{plain}
\bibliography{main}

@inproceedings{hazan2017spectral,
  title     = {Learning Linear Dynamical Systems via Spectral Filtering},
  author    = {Elad Hazan and Karan Singh and Cyril Zhang},
  booktitle = {Advances in Neural Information Processing Systems 30},
  pages     = {6702--6712},
  year      = {2017}
}

@inproceedings{hazan2018generallds,
  title     = {Spectral Filtering for General Linear Dynamical Systems},
  author    = {Elad Hazan and Holden Lee and Karan Singh and Cyril Zhang and Yi Zhang},
  booktitle = {Advances in Neural Information Processing Systems 31},
  pages     = {4639--4648},
  year      = {2018}
}

@article{hardt2018gradient,
  title={Gradient descent learns linear dynamical systems},
  author={Hardt, Moritz and Ma, Tengyu and Recht, Benjamin},
  journal={Journal of Machine Learning Research},
  volume={19},
  number={29},
  pages={1--44},
  year={2018}
}

@inproceedings{ghai2020marginallystable,
  title     = {No-Regret Prediction in Marginally Stable Systems},
  author    = {Udaya Ghai and Holden Lee and Karan Singh and Cyril Zhang and Yi Zhang},
  booktitle = {Proceedings of the 33rd Conference on Learning Theory},
  series    = {Proceedings of Machine Learning Research},
  volume    = {125},
  pages     = {1714--1757},
  year      = {2020}
}

@inproceedings{rashidinejad2020slip,
  title     = {SLIP: Learning to Predict in Unknown Dynamical Systems with Long-Term Memory},
  author    = {Paria Rashidinejad and Jiantao Jiao and Stuart J. Russell},
  booktitle = {Advances in Neural Information Processing Systems 33},
  year      = {2020}
}

@article{tsiamis2023onlinekalman,
  title   = {Online Learning of the Kalman Filter With Logarithmic Regret},
  author  = {Anastasios Tsiamis and George J. Pappas},
  journal = {IEEE Transactions on Automatic Control},
  volume  = {68},
  number  = {5},
  pages   = {2774--2789},
  year    = {2023}
}

@article{marsden2025dimensionfree,
  title   = {Dimension-free Regret for Learning Asymmetric Linear Dynamical Systems},
  author  = {Annie Marsden and Elad Hazan},
  journal = {arXiv preprint arXiv:2502.06545},
  year    = {2025}
}

@inproceedings{simchowitz2019semiparametric,
  title     = {Learning Linear Dynamical Systems with Semi-Parametric Least Squares},
  author    = {Max Simchowitz and Ross Boczar and Benjamin Recht},
  booktitle = {Proceedings of the 32nd Conference on Learning Theory},
  series    = {Proceedings of Machine Learning Research},
  volume    = {99},
  pages     = {2714--2802},
  year      = {2019}
}

@inproceedings{tsiamis2019stochsysid,
  title     = {Finite Sample Analysis of Stochastic System Identification},
  author    = {Anastasios Tsiamis and George J. Pappas},
  booktitle = {2019 IEEE 58th Conference on Decision and Control (CDC)},
  pages     = {3648--3654},
  year      = {2019},
  organization = {IEEE}
}

@inproceedings{lee2022improved,
  title     = {Improved Rates for Prediction and Identification of Partially Observed Linear Dynamical Systems},
  author    = {Holden Lee},
  booktitle = {Proceedings of the 33rd International Conference on Algorithmic Learning Theory},
  series    = {Proceedings of Machine Learning Research},
  volume    = {167},
  pages     = {668--698},
  year      = {2022}
}

@inproceedings{bakshi2023newapproach,
  title     = {A New Approach to Learning Linear Dynamical Systems},
  author    = {Ainesh Bakshi and Allen Liu and Ankur Moitra and Morris Yau},
  booktitle = {Proceedings of the 55th Annual ACM Symposium on Theory of Computing},
  pages     = {335--348},
  year      = {2023}
}

@article{gatermann1992faber,
  title={Explicit Faber polynomials on circular sectors},
  author={Gatermann, Karin and Hoffmann, Christoph and Opfer, Gerhard},
  journal={Mathematics of Computation},
  volume={58},
  number={197},
  pages={241--253},
  year={1992},
  doi={10.1090/S0025-5718-1992-1106967-4}
}

@inproceedings{fattahi2021logsamples,
  title     = {Learning Partially Observed Linear Dynamical Systems from Logarithmic Number of Samples},
  author    = {Salar Fattahi},
  booktitle = {Proceedings of the 3rd Conference on Learning for Dynamics and Control},
  series    = {Proceedings of Machine Learning Research},
  volume    = {144},
  pages     = {60--72},
  year      = {2021}
}

@article{coleman1987faber,
  author  = {John P. Coleman and Russell A. Smith},
  title   = {The Faber Polynomials for Circular Sectors},
  journal = {Mathematics of Computation},
  volume  = {49},
  number  = {180},
  pages   = {231--241},
  year    = {1987}
}

@inproceedings{marsdenuniversal,
  title={Universal Sequence Preconditioning},
  author={Marsden, Annie and Hazan, Elad},
  booktitle={The Thirty-ninth Annual Conference on Neural Information Processing Systems},
  year={2025}
}

@article{hazan2026spectralfilteringcomplex,
  title   = {Spectral Filtering for Complex Linear Dynamical Systems},
  author  = {Elad Hazan and Annie Marsden},
  journal = {arXiv preprint arXiv:2601.22400},
  year    = {2026}
}

@article{rissanen1978modeling,
  title={Modeling by shortest data description},
  author={Rissanen, Jorma},
  journal={Automatica},
  volume={14},
  number={5},
  pages={465--471},
  year={1978},
  publisher={Elsevier}
}

@book{grunwald2007minimum,
  title={The minimum description length principle},
  author={Gr{\"u}nwald, Peter D},
  year={2007},
  publisher={MIT press}
}

@article{vovk2001competitive,
  title={Competitive on-line statistics},
  author={Vovk, Volodya},
  journal={International Statistical Review},
  volume={69},
  number={2},
  pages={213--248},
  year={2001},
  publisher={Wiley Online Library}
}

@article{azoury2001relative,
  title={Relative loss bounds for on-line density estimation with the exponential family of distributions},
  author={Azoury, Katy S and Warmuth, Manfred K},
  journal={Machine Learning},
  volume={43},
  number={3},
  pages={211--246},
  year={2001},
  publisher={Springer}
}

@article{hazan2007logarithmic,
  title={Logarithmic regret algorithms for online convex optimization},
  author={Hazan, Elad and Agarwal, Amit and Kale, Satyen},
  journal={Machine Learning},
  volume={69},
  number={2-3},
  pages={169--192},
  year={2007},
  publisher={Springer}
}

@book{cesa2006prediction,
  title={Prediction, learning, and games},
  author={Cesa-Bianchi, Nicol{\`o} and Lugosi, G{\'a}bor},
  year={2006},
  publisher={Cambridge university press}
}

@book{hazan2016introduction,
  title={Introduction to online convex optimization},
  author={Hazan, Elad},
  journal={Foundations and Trends{\textregistered} in Optimization},
  volume={2},
  pages={157--325},
  year={2016},
  publisher={Now Publishers, Inc.}
}

@article{dogariu2025universal,
  title={Universal learning of nonlinear dynamics},
  author={Dogariu, Evan and Brahmbhatt, Anand and Hazan, Elad},
  journal={arXiv preprint arXiv:2508.11990},
  year={2025}
}

@book{kailath1980linear,
  title={Linear Systems},
  author={Kailath, Thomas},
  year={1980},
  publisher={Prentice-Hall}
}

@book{chen1999linear,
  title={Linear System Theory and Design},
  author={Chen, Chi-Tsong},
  edition={3rd},
  year={1999},
  publisher={Oxford University Press}
}

\newpage

\appendix

\section{Coefficient growth of extremal polynomials}
\label{appendix:coeff_growth}
The following lemma shows that any monic polynomial that ensures all values in $[-1,1]$ are exponentially small, must have exponentially large coefficients. 

\begin{theorem}[Coefficient growth for flat monic polynomials]
Fix $M>0$. Let $p(x)=\sum_{j=0}^n a_j x^j,
 a_n=1$ , be a monic polynomial of degree $n$. If
\[
\max_{x\in[-1,1]} |p(x)| \leq M 2^{1-n},
\]
then
\[
\max_{0\leq j\leq n} |a_j|
\geq  \frac{1}{n} 2^{\Omega(n)} 
\]
In particular, at least one coefficient of $p$ is exponentially large in $n$.
\end{theorem}

\begin{proof}
We split the proof into two parts. The first part is the conceptual reduction,
and the second part is the technical no-cancellation estimate.

\paragraph{Part I: Reduction to a no-cancellation statement.}

Write $p$ in the Chebyshev basis:
\[
p(x)=\sum_{k=0}^n b_k T_k(x),
\]
where $T_k$ is the $k$th Chebyshev polynomial. Since
\[
T_n(x)=2^{n-1}x^n+\text{lower-degree terms},
\]
and $p$ is monic, the coefficient of $T_n$ must be
\[
b_n=2^{1-n}.
\]

The remaining Chebyshev coefficients are controlled by the assumption that
$p$ is small on $[-1,1]$. Indeed, using $T_k(\cos\theta)=\cos(k\theta)$ and
orthogonality on $[0,\pi]$,
\begin{eqnarray*}
|b_k| & =  \left| \frac{2}{\pi}\int_0^\pi p(\cos\theta)\cos(k\theta)\,d\theta \right| \\
& \leq  \frac{2}{\pi}\int_0^\pi \left| p(\cos\theta) \right| \left| \cos(k\theta)  \right| d\theta \\
& \leq  \frac{2}{\pi}\int_0^\pi \left| p(\cos\theta) \right|  d\theta \\
& \leq 2 \max_{x \in [-1,1] } |p(x) | \leq 2M 2^{1-n} = 4M 2^{-n}.
\end{eqnarray*}
Thus, after scaling by $2^{n-1}$, the polynomial
\[
q(x):=2^{n-1}p(x)
\]
has the form
\[
q(x)=T_n(x)+\sum_{k=0}^{n-1} c_k T_k(x),
\qquad |c_k|\leq 2M.
\]

So the whole issue is the following: can the bounded lower-degree Chebyshev
terms cancel the coefficient growth of $T_n$? The answer is no. This is the
content of the next lemma.

\begin{lemma}[No cancellation, asymptotic form]
Fix $M>0$. Suppose
\[
q(x)=T_n(x)+\sum_{k=0}^{n-1}c_kT_k(x),
\qquad |c_k|\leq 2M.
\]
Write
\[
q(x)=\sum_{j=0}^n d_jx^j.
\]
Then
\[
\max_{0\leq j\leq n}|d_j|
\geq
\frac{1}{n+1}2^{(1+\Omega_M(1))n}.
\]
Equivalently, the coefficients of $q$ grow like $(2+\Omega_M(1))^n/n$.
\end{lemma}

\begin{proof}
Choose $\alpha>1$ sufficiently large depending only on $M$, and set
\[
y:=\frac{\alpha-\alpha^{-1}}{2}.
\]
For a polynomial $r(x)=\sum_j r_jx^j$, define
\[
\|r\|_y:=\sum_j |r_j|y^j.
\]

We use the standard Chebyshev coefficient estimate
\[
\|T_k\|_y
=
\frac{\alpha^k+(-1)^k\alpha^{-k}}{2}
=
\Theta(\alpha^k).
\]
Thus
\[
\|T_n\|_y=\Omega(\alpha^n),
\qquad
\|T_k\|_y\leq O(\alpha^k).
\]
Therefore
\[
\begin{aligned}
\|q\|_y
&\geq
\|T_n\|_y
-
\sum_{k=0}^{n-1}|c_k|\|T_k\|_y \\
&\geq
\Omega(\alpha^n)
-
O_M\left(\sum_{k=0}^{n-1}\alpha^k\right).
\end{aligned}
\]
Since
\[
\sum_{k=0}^{n-1}\alpha^k
\leq
\frac{\alpha^n}{\alpha-1},
\]
choosing $\alpha$ large enough depending on $M$ gives
\[
\|q\|_y\geq \Omega_M(\alpha^n).
\]

On the other hand, if
\[
D:=\max_{0\leq j\leq n}|d_j|,
\]
then
\[
\|q\|_y
=
\sum_{j=0}^n |d_j|y^j
\leq
D(n+1)y^n.
\]
Hence
\[
D
\geq
\frac{1}{n+1}\Omega_M\left(\left(\frac{\alpha}{y}\right)^n\right).
\]
Finally,
\[
\frac{\alpha}{y}
=
\frac{2\alpha^2}{\alpha^2-1}
>2.
\]
Therefore
\[
D
\geq
\frac{1}{n+1}2^{(1+\Omega_M(1))n}.
\]
\end{proof}

\end{proof}
 
\paragraph{A constant-rate conjecture.}
The preceding theorem proves exponential coefficient growth at the sharp
Chebyshev scale $2^{-n}$. Its proof uses this scale essentially: after
normalizing by the leading Chebyshev coefficient, the lower-degree Chebyshev
coefficients remain uniformly bounded. For a weaker bound such as $\exp(-cn)$
with $c<\log 2$, the same normalization allows the lower-degree Chebyshev
coefficients to be exponentially large, so the no-cancellation argument no
longer applies.

We nevertheless expect the qualitative obstruction to persist.

\begin{conjecture}[Coefficient growth at any exponential rate]
For every $c>0$ there exists $c'>0$ such that the following holds for all
sufficiently large $n$. If
\[
p_n(x)=x^n+\sum_{j=0}^{n-1}a_jx^j
\]
is a monic polynomial satisfying
\[
\max_{x\in[-1,1]} |p_n(x)| \le e^{-cn},
\]
then
\[
\max_{0\le j\le n}|a_j| \ge e^{c'n}.
\]
\end{conjecture}

\section{Proof of Lemma \ref{lem:faber_wedge_bounds}}
\label{appendix:faber}

Recall that Lemma \ref{lem:faber_wedge_bounds} asserts that for every $\delta\in(0,\pi)$ and $n\ge1$, there is a real monic polynomial
\[
  p_{n,\delta}(z)=z^n+\sum_{i=1}^n c_i z^{n-i}
\]
with $c_0:=1$ such that
\[
  \sup_{z\in W_{\pi-\delta}}|p_{n,\delta}(z)|
  \le
  2e^{-n\delta^2/\pi^2},
  \qquad
  \max_{0\le i\le n}|c_i|\le 100^n .
\]

\begin{proof}
\medskip\noindent\textbf{Part I: Exponentially small norm on the wedge.}

For $\theta\in(0,\pi)$, define
\[
W_\theta=\{z\in\mathbb{C}: |z|\le 1,\ |\arg(z)|\le \theta\}.
\]
We use the following sector-specific Faber estimate for circular sectors, due
to Coleman and Smith~\cite{coleman1987faber} and stated in the convenient
normalization of Gatermann, Hoffmann, and Opfer~\cite{gatermann1992faber}.
Let $F_{n,\theta}$ be the monic degree-$n$ Faber polynomial of
$W_\theta$, and let $\rho(\theta)$ be the transfinite diameter of
$W_\theta$. Coleman and Smith compute
\[
\rho(\theta)=t^{-t}(2-t)^{t-2},
\qquad
 t=\frac{\theta}{\pi}.
\]
In the notation of~\cite[Eq.~(2.6)]{gatermann1992faber}, the scaled Faber
polynomial is
\[
\widetilde F_{n,\theta}(z)=\rho(\theta)^{-n}F_{n,\theta}(z).
\]
Moreover, \cite[Eq.~(3.10)]{gatermann1992faber} gives
\[
\sup_{z\in W_\theta}|\widetilde F_{n,\theta}(z)|\le2.
\]
Consequently,
\[
\sup_{z\in W_\theta}|F_{n,\theta}(z)|
\le
2\rho(\theta)^n .
\]

Set
\[
\theta=\pi-\delta,
\qquad
\eta=\frac{\delta}{\pi}.
\]
Then
\[
t=\frac{\theta}{\pi}=1-\eta,
\]
and therefore
\[
\rho(\pi-\delta)
=
(1-\eta)^{-(1-\eta)}
(1+\eta)^{-(1+\eta)}.
\]
Thus
\[
\log\!\left(\frac{1}{\rho(\pi-\delta)}\right)
=
(1+\eta)\log(1+\eta)
+
(1-\eta)\log(1-\eta).
\]
For $0\le\eta<1$,
\[
(1+\eta)\log(1+\eta)+(1-\eta)\log(1-\eta)
=
\sum_{k=1}^{\infty}
\frac{2\eta^{2k}}{(2k)(2k-1)}
\ge
\eta^2.
\]
Hence
\[
\rho(\pi-\delta)
\le
 e^{-\eta^2}
=
e^{-\delta^2/\pi^2}.
\]
Consequently,
\[
\sup_{z\in W_{\pi-\delta}} |F_{n,\pi-\delta}(z)|
\le
2e^{-n\delta^2/\pi^2}.
\]

Finally, $W_{\pi-\delta}$ is invariant under complex conjugation. Therefore,
if necessary, replacing $F_{n,\pi-\delta}$ by
\[
p_{n,\delta}(z)
=
\frac{
F_{n,\pi-\delta}(z)
+
\overline{F_{n,\pi-\delta}(\overline z)}
}{2}
\]
preserves monicity, gives real coefficients, and does not increase the
supremum norm on $W_{\pi-\delta}$. Thus we may take $p_{n,\delta}$ to be
real and monic, with
\[
\sup_{z\in W_{\pi-\delta}} |p_{n,\delta}(z)|
\le
2e^{-n\delta^2/\pi^2}.
\]

\medskip\noindent\textbf{Part II: Coefficient bound.}

Write
\[
p_{n,\delta}(z)=z^n+\sum_{i=1}^n c_i z^{n-i},
\qquad c_0=1.
\]
Since $[0,1]\subseteq W_{\pi-\delta}$, Part I gives
\[
\sup_{x\in[0,1]} |p_{n,\delta}(x)|\le 2.
\]
Define
\[
q(x)=p_{n,\delta}\!\left(\frac{x+1}{2}\right).
\]
Then
\[
\sup_{x\in[-1,1]} |q(x)|\le 2.
\]

We use the following elementary fact: if a degree-$n$ polynomial $q$ satisfies
$\sup_{x\in[-1,1]}|q(x)|\le2$, then the $\ell_1$ norm of its coefficient is at
most $24^n$. To see this, write $q$ in the Chebyshev basis
\[
q(x)=\sum_{k=0}^{n}b_kT_k(x).
\]
By the usual Chebyshev orthogonality formulas,
\[
b_0=\frac1\pi\int_0^\pi q(\cos\theta)\,d\theta,
\qquad
b_k=\frac2\pi\int_0^\pi q(\cos\theta)\cos(k\theta)\,d\theta
\quad (k\ge1).
\]
Since $\sup_{x\in[-1,1]}|q(x)|\le2$, it follows that
\[
|b_k|\le4,
\qquad 0\le k\le n.
\]
Also, the recurrence
\[
T_{k+1}(x)=2xT_k(x)-T_{k-1}(x)
\]
implies that the monomial coefficient $\ell_1$ norm of $T_k$ is at most
$3^k$. Hence
\[
\|q\|_1
\le
\sum_{k=0}^{n}|b_k|\,\|T_k\|_1
\le
4(n+1)3^n
\le
24^n .
\]

Finally,
\[
p_{n,\delta}(z)=q(2z-1).
\]
Substituting $2z-1$ increases the monomial coefficient $\ell_1$ norm by at
most $3^n$, since
\[
\|(2z-1)^j\|_1=3^j\le3^n
\qquad
\text{for all } j\le n.
\]
Therefore
\[
\sum_{i=0}^{n}|c_i|
=
\|p_{n,\delta}\|_1
\le
3^n\|q\|_1
\le
72^n
\le
100^n.
\]
In particular,
\[
\max_{0\le i\le n}|c_i|\le100^n.
\]
Combining Parts I and II proves the lemma.
\end{proof}

\end{document}